\documentclass{article}
\usepackage{ijcai26}
\usepackage[normalem]{ulem}
\usepackage{times}
\usepackage{soul}
\usepackage{url}
\usepackage{bbm}
\usepackage{xcolor}
\usepackage{tabularx} 
\usepackage{booktabs}
\usepackage{amssymb}
\usepackage[hidelinks]{hyperref}
\usepackage[utf8]{inputenc}
\usepackage[small]{caption}
\usepackage{graphicx}
\usepackage{amsmath}
\usepackage{amsthm}
\usepackage{booktabs}
\usepackage{algorithm}
\usepackage{algorithmic}
\usepackage[switch]{lineno}

\theoremstyle{plain}  
\newtheorem{proposition}{Proposition}

\newtheorem{theorem}{Theorem}

\title{Which Graph Shift Operator? A Spectral Answer to an Empirical Question}

\author{
    Yassine Abbahaddou
    \affiliations
    Ecole Polytechnique, IP Paris
    \emails
    yassine.abbahaddou@polytechnique.edu
}

\begin{document}

\maketitle

\begin{abstract}
Graph Neural Networks (GNNs) have established themselves as the leading models for learning on graph-structured data, generally categorized into spatial and spectral approaches. Central to these architectures is the Graph Shift Operator (GSO), a matrix representation of the graph structure used to filter node signals. However, selecting the optimal GSO, whether fixed or learnable, remains largely empirical. In this paper, we introduce a novel alignment gain metric that quantifies the geometric distortion between the input signal and label subspaces. Crucially, our theoretical analysis connects this alignment directly to generalization bounds via a spectral proxy for the Lipschitz constant. This yields a principled, computation-efficient criterion to rank and select the optimal GSO for any prediction task prior to training, eliminating the need for extensive search. Source code to reproduce all experiments and the MSD metric is available at \url{https://github.com/abbahaddou/MSD-Alignment}.
\end{abstract}

\section{Introduction}\label{sec:introduction}

Graph Neural Networks (GNNs) have established themselves as the ubiquitous framework for learning on non-Euclidean data, driving breakthroughs in domains as diverse as computational biology, social network analysis, and recommender systems \cite{zhang2021graph,malitesta2025fair}. Their ability to capture complex relational dependencies has made them the de facto standard for modeling graph-structured data. This rapid expansion has primarily bifurcated into two research paradigms: spatial-based methods \cite{geisler2024spatio}, which define convolutions directly in the vertex domain by aggregating features from local neighborhoods , and spectral-based methods \cite{xin2024chebynet}, which utilize graph signal processing theory to filter signals via spectral decomposition.

Despite their methodological differences, both paradigms rely fundamentally on the same computational engine: the Graph Shift Operator (GSO). Formally, a GSO is a matrix representation of the graph structure used to filter node signals. While classical implementations employ fixed operators, such as the normalized Laplacian for high-pass filtering or the random walk Laplacian for low-pass filtering, recent advances have introduced polynomial-based GSOs, e.g., ChebNet, JacobiConv, to approximate ideal filters without expensive eigen-decompositions.

While the GSO is central to GNN architecture, a critical open question remains: \emph{What is the optimal GSO for a specific prediction task?} Currently, identifying the correct operator is  largely dominated by empirical trial-and-error or the use of learnable coefficients that linearly combine powers of the operator. These approaches are often computationally inefficient and lack theoretical grounding. They treat the graph structure as a black box, failing to analyze how the geometry of the input data interacts with the geometry of the target task.

To overcome these limitations, it is essential to establish a mechanism that can identify the most effective GSO pre-training. Without such a criterion, the selection process becomes an expensive, multi-stage endeavor where models must be fully trained on each candidate operator, evaluated on test data, and only then compared at the inference stage to see which performed best. This post-hoc selection is not only resource-intensive but also fails to provide any insight into why a specific graph structure aligns with the task

In this paper, we propose a paradigm shift from empirical selection to geometric alignment. We posit that the effectiveness of a GSO is determined by how well it aligns the subspace of the input node features with the subspace of the target labels. This hypothesis is grounded in emerging learning theory which suggests that representation alignment is a precursor to generalization. Our Contributions could be summarized as follows,
\begin{enumerate}
    \item \textbf{Metric.} We introduce an alignment gain metric that quantifies the geometric distortion between the input signal and label subspaces. Specifically, we utilize the Generalized Courant-Fischer Theorem to measure the \emph{Distance Mapping Distortion} between input and label manifolds, providing a robust proxy for task compatibility without requiring training.

    \item \textbf{Theory.} We provide a theoretical analysis connecting this alignment to generalization bounds. We prove that minimizing the spectral distortion (maximizing alignment) theoretically guarantees a tighter bound on the generalization error by lowering the effective Lipschitz constant of the learning mapping.

    \item \textbf{Application.} This yields a principled, training-free criterion to rank GSOs, allowing researchers to identify the optimal graph structure for a given task without extensive empirical search.
\end{enumerate}

\section{Related and Background}\label{sec:background}

\textbf{Notations.} We consider a graph $\mathcal{G} = (\mathcal{V}, \mathcal{E})$, where $\mathcal{V} =\{1, \ldots, \mathrm{N}\}$ is the set of nodes, and $\mathcal{E} \subset \mathcal{V}\times \mathcal{V}$ is the set
of edges. The graph structure is characterized by an adjacency matrix $\mathbf{A} \in \mathbb{R}^{\mathrm{N} \times \mathrm{N}}$ and a diagonal degree matrix $\mathbf{D}$ where $\mathbf{D}_{ii} = \sum_j \mathbf{A}_{ij}$. The node features and labels are defined as collections of samples $\mathrm{X} = (\mathrm{x}_i)_{i=1}^{\mathrm{N}} \in \mathbb{R}^{N \times d}$ and $\mathrm{Y} = (\mathrm{y}_i)_{i=1}^{\mathrm{N}}$, where $\mathbf{x}_i$ and $\mathbf{y}_i$ correspond to the feature vector and label of node $i$, respectively. We assume the pairs $(\mathbf{x}_i, \mathbf{y}_i)$ are distributed according to a joint distribution $\mathbb{P}_{\mathrm{XY}}$, with $\mathbb{P}_{\mathrm{X}}$ and $\mathbb{P}_{\mathrm{Y}}$ denoting the corresponding marginal distributions over the feature and label spaces.

\subsection{Graph Shift Operatores}
The fundamental engine of any spectral or spatial GNN is the Graph Shift Operator (GSO) \cite{mateos2019connecting,abbahaddou2024centrality}. It serves as the algebraic representation of the graph's topology, dictating how information diffuses across the vertex domain. Formally, for a graph $\mathcal{G}$, a GSO is a matrix $\mathbf{S} = (\mathrm{S}_{ij} )_{ij} \in \mathbb{R}^{\mathrm{N} \times \mathrm{N}}$ satisfying $\mathrm{S}_{ij} \neq 0$ if and only if $i \neq j$ and $(i, j) \in \mathcal{E}$.
The motivation for selecting a specific GSO lies in its spectral properties: different GSOs filter node signals differently, acting as the geometric lens through which the model views the data.

Classical operators are categorized by their spectral effects: high-pass filters, such as the symmetric normalized Laplacian $\mathbf{L}_{\text{sym}} = \mathbf{I} - \mathbf{D}^{-1/2} \mathbf{A} \mathbf{D}^{-1/2}$, emphasize local variations and edge details. Conversely, low-pass filters, like the random walk Laplacian $\mathbf{L}_{\text{rw}} = \mathbf{I} - \mathbf{D}^{-1} \mathbf{A}$ or renormalized adjacency $\mathbf{\hat{A}} = \mathbf{D}^{-1/2} \mathbf{A} \mathbf{D}^{-1/2}$ \cite{kipf2017semisupervised}, smooth signals to retain global structural patterns. The raw adjacency matrix $\mathbf{A}$ typically acts as a full-pass filter. A comprehensive summary and mathematical comparison of these classical GSOs can be found in the Appendix \ref{app:class_Gso}. Selecting the correct operator is critical, as an suboptimal GSO can lead to either over-smoothing or the loss of discriminative high-frequency information.
 
Beyond fixed GSOs, recent work has introduced flexible parameterizations to adapt the GSO to the task. Approaches like PGSO \cite{dasoulaslearning} and PD-GCN \cite{luflexible} allow the GSO to be learned dynamically, for instance, by optimizing diffusion scopes or combining multiple Laplacian forms. 

The selection of an optimal Graph Shift Operator (GSO) is critical because it dictates how well the input feature subspace aligns with the target label subspace, acting as the geometric lens through which the model views data. By minimizing spectral distortion, a GSO can effectively \emph{rewire} the graph to ensure information is aggregated only from task-relevant neighbors rather than being constrained by the raw adjacency. This alignment is fundamentally linked to performance, as maximizing it theoretically guarantees a tighter bound on the model's generalization error.

\textbf{Notation Note:} In the following sections, we introduce specific notation to distinguish between structural and task-specific geometries.  To distinguish between different types of operators, we use standard math font (e.g., $L_Z, L_Y$) for operators derived from manifold approximations and specific task geometries. We reserve \textbf{boldface} notation, such as $\mathbf{A}$ or $\mathbf{L}$, for general operators that may deviate from the original graph's sparsity, such as those used in manifold reconstruction.
\subsection{Graph Neural Networks}

Graph Neural Networks (GNNs) operate on graph-structured data defined by the tuple $(\mathcal{G}, \mathrm{X})$ \cite{brody2021attentive,hu2020gpt}. They process this data through a sequence of computational layers, generating hidden node representations $\mathrm{H}^{(i)} \in \mathbb{R}^{\mathrm{N} \times \mathrm{d}_i}$ at each layer $i$. Initialized with the raw node features $\mathrm{H}^{(0)} = \mathrm{X}$, the dominant paradigm for these updates is the Message Passing Neural Network (MPNN) framework, which decomposes the layer-wise computation into two distinct phases,

\textbf{Message Passing (Aggregation).} This step aggregates information from a node's local neighborhood to capture structural context. In matrix formulation, this corresponds to filtering the node signals using the Graph Shift Operator $\mathbf{S}$,
\begin{equation}\mathrm{M}^{(i+1)} = \mathbf{S} \mathrm{H}^{(i)}, \label{cgsO:eq:message_passing}\end{equation}where $\mathbf{S} \in \mathbb{R}^{\mathrm{N} \times \mathrm{N}}$ is the chosen GSO, e.g., Laplacian, Adjacency, effectively diffusing information across edges.

\textbf{Update (Transformation).} Following aggregation, the diffused signals are transformed into the next layer's representation using learnable parameters and non-linear activation functions,

\begin{equation}\mathrm{H}^{(i+1)} = \sigma \left( \mathrm{M}^{(i+1)} \mathrm{W}^{(i)} \right), \label{cgsO:eq:update}\end{equation}

where $\mathrm{W}^{(i)} \in \mathbb{R}^{\mathrm{d}_{i} \times \mathrm{d}_{i+1}}$ is the learnable weight matrix for layer $i$, and $\sigma(\cdot)$ denotes a non-linear activation function, e.g., ReLU.

Formally, we denote the entire GNN function, parameterized by weights $\theta$ and governed by the operator $\mathbf{S}$, as $f_\theta(\cdot; \mathbf{S})$.  we note that standard architectures are essentially specific instantiations of $\mathbf{S}$. For instance, GCN \cite{kipf2017semisupervised} utilizes the renormalized adjacency $\mathbf{S} = \tilde{\mathbf{D}}^{-\frac{1}{2}}\tilde{\mathbf{A}}\tilde{\mathbf{D}}^{-\frac{1}{2}}$ to smooth representations, while GIN \cite{xu2018how} employs $\mathbf{S} = \mathbf{A} + (1+\epsilon)\mathbf{I}$ coupled with an MLP update to maximize discriminative power. Thus, distinct GNN backbones can be viewed primarily as variations in their GSOs

In this work, we focus on the node classification task, where the objective is to predict labels for nodes based on graph structure and features. Given the GNN backbone $f_\theta(\cdot; \mathbf{S})$, the optimal parameters are learned by minimizing the expected risk over the data distribution. The loss function $\mathcal{L}$ is formally defined as,

\begin{equation}
    \mathcal{L}(\theta, \mathbf{S}) = \mathbb{E}_{(\mathrm{X},\mathrm{Y}) \sim \mathbb{P}_{\mathrm{XY}}} \left[\ell(f_\theta(\mathrm{X}; \mathbf{S}), \mathrm{Y})\right], \label{cgsO:eq:loss}
\end{equation}

where $\ell(\cdot, \cdot)$ denotes a standard objective function, e.g., Cross-Entropy, measuring the discrepancy between the predicted logits and the ground truth labels $\mathrm{Y}$.

\textbf{Generalization vs. Empirical Risk.} The loss defined in Eq. \eqref{cgsO:eq:loss} represents the ideal objective, commonly referred to as the \textit{generalization error}. It quantifies the model's performance on the true, albeit unknown, joint distribution:
\begin{equation}\mathcal{E}_{gen}(f_\theta) \triangleq \mathbb{E}_{(\mathrm{X},\mathrm{Y}) \sim \mathbb{P}_{\mathrm{XY}}} \left[\ell(f_\theta(\mathrm{X}; \mathbf{S}), \mathrm{Y})\right].\label{cgsO:eq:gen_error}\end{equation}

In practice, the joint distribution $\mathbb{P}_{\mathrm{XY}}$ is inaccessible. Instead, we operate under a semi-supervised setting where we are provided with a training mask $\mathcal{V}_{\text{train}} \subset \mathcal{V}$ containing the indices of labeled nodes. Consequently, we approximate the generalization error using the \textit{empirical error}, computed as the finite average loss over the observed training set,
\begin{equation}
\hat{\mathcal{E}}_{emp}(f_\theta) = \frac{1}{|\mathcal{V}_{\text{train}}|} \sum_{i \in \mathcal{V}_{\text{train}}} \ell \left( f_\theta(\mathrm{X}; \mathbf{S})_i, \mathrm{y}_i \right),\label{cgsO:eq:emp_error}\end{equation}where $f_\theta(\mathrm{X}; \mathbf{S})_i$ denotes the predicted output for node $i$. The central goal of learning is to find parameters $\hat{\theta}$ that minimize $\hat{\mathcal{E}}_{emp}(f_\theta)$ while ensuring that the \textit{generalization gap} $|\mathcal{E}_{gen} - \hat{\mathcal{E}}_{emp}|$ remains bounded. In the context of GNNs, this gap is strictly controlled by the alignment between the graph structure (encoded by $\mathbf{S}$) and the underlying data geometry.

\section{Geometric Framework for GSO Selection}

In this section, we formalize the problem of GSO selection and introduce our geometric framework. We propose viewing the learning process as a mapping between manifolds and define a spectral metric to quantify the alignment between the input graph structure and the target task.

Given a fixed Graph Shift Operator (GSO) $\mathbf{S}$, a GNN $f_\theta(\cdot; \mathbf{S})$ is trained to minimize the expected loss $\mathcal{L}(\theta, \mathbf{S}) = \mathbb{E}_{(X,Y) \sim \mathbb{P}_{XY}} [\ell(f_\theta(X; \mathbf{S}), Y)]$. Our goal is to identify the optimal operator $\mathbf{S}^*$ that minimizes this loss,

\begin{equation}
    \mathbf{S}^* = \arg \min_{\mathbf{S}} \min_{\theta} \mathcal{L}(\theta, \mathbf{S}).
\end{equation}

The operator $\mathbf{S}^*$ denotes the GSO that aligns perfectly with the downstream task. Unlike the predefined graph structure $\mathcal{G}$, which may be noisy or incomplete, $\mathbf{S}^*$ encodes the optimal diffusion pathways for the signal. Consequently, identifying $\mathbf{S}^*$ reveals exactly which operator the GNN should be trained on to maximize performance, effectively "rewiring" the graph to ensure information is aggregated only from task-relevant neighbors rather than being constrained by the raw adjacency.

Finding $\mathbf{S}^*$ typically requires expensive training for every candidate $\mathbf{S}$. To avoid this, we seek a proxy objective based on subspace alignment. We seek an alignment metric $\mathcal{A}(\cdot, \cdot)$ that quantifies the similarity between the column space of the filtered signals, e.g., $\mathbf{S}X$, and the label space $Y$. We define the alignment gain as the improvement in this metric when applying the GSO,
\begin{equation}
    \Delta\mathcal{A}(\mathbf{S}, X, Y) := \mathcal{A}(\mathbf{S}X, Y) - \mathcal{A}(X, Y).
\end{equation}

Designing a robust alignment metric $\mathcal{A}$ is non-trivial due to several geometric challenges,

\begin{itemize}
    \item \textbf{Dimensionality Mismatch:} The input feature space $X$ (or hidden representations) and the label space $Y$ often reside in different ambient dimensions ($d_{in} \neq d_{out}$). This discrepancy make standard metrics like principal angles distance inapplicable, as they strictly require both subspaces to reside within the same ambient space for meaningful comparison.

    \item \textbf{Manifold Structure:} $\mathrm{X}$ and $\mathrm{Y}$ are not just linear subspaces but samples from underlying manifolds. A linear alignment metric might fail to capture the non-linear "distortion" or topological changes induced by the GSO.

    \item \textbf{Generalization Link:} The metric must not only measure geometric overlap but also correlate theoretically with the model's ability to generalize to unseen nodes.To address these challenges, we move beyond simple subspace angles and adopt a spectral distortion perspective.
\end{itemize}

We begin by establishing a discrete representation of the data manifolds through Manifold Approximation via Graph Laplacians (Section \ref{subsec:manifold_approx}). Building upon these discrete operators, we then formalize the Spectral Distortion Metric (Section \ref{subsec:metric_def}), which quantifies the geometric alignment between input and target spaces. Finally, we bridge the gap between geometry and performance in Theoretical Analysis: Generalization Guarantees (Section \ref{subsec:generalization}), where we prove that minimizing this distortion directly tightens the generalization error bounds of the GNN.

\subsection{Manifold Approximation via Graph Laplacians}\label{subsec:manifold_approx}
To quantify alignment, we approximate the geometry of the signal space and the target space using discrete graph structures. We denote the signal matrix as $\mathbf{\mathrm{Z}} \in \mathbb{R}^{N \times d}$, which may represent either the raw input features $\mathrm{X}$, or  the diffused signals $\mathbf{S}\mathrm{X}$, depending on the context.

\textbf{Input Manifold $\mathcal{G}_{\mathrm{Z}}$.} We capture the intrinsic local geometry of $\mathbf{Z}$ by constructing a symmetrized $k$-Nearest Neighbor ($k$-NN) graph. We first compute the pairwise Euclidean distances between node representations. To ensure robustness and computational efficiency, we define the adjacency matrix $\mathcal{W}_{Z} \in \{0, 1\}^{N \times N}$ as a binary, undirected operator. Specifically, an edge exists between nodes $i$ and $j$ if either node is among the $k$-nearest neighbors of the other. Mathematically, this symmetrization is formalized as:
\begin{equation}
\mathcal{W}_{\mathbf{Z}, ij} = \max\left\{ \mathbbm{1}(j \in \mathcal{N}_k(i)), \mathbbm{1}(i \in \mathcal{N}_k(j)) \right\},
\end{equation}
where $\mathcal{N}_k(i)$ denotes the set of $k$ nearest neighbors of node $i$ based on the $L_2$ distance, and $\mathbbm{1}(\cdot)$ is the indicator function. This construction effectively discretizes the manifold structure into a sparse, symmetric adjacency matrix.

\textbf{Output Manifold $\mathcal{G}_{\mathrm{Y}}$.} Similarly, we define the target geometry using the labels  $\mathrm{Y}$. The output graph $\mathcal{G}_{\mathrm{Y}}$ encodes the ideal task structure, where nodes belonging to the same class are densely connected. The adjacency $\mathcal{W}_{\mathrm{Y}}$ is defined such that $\mathcal{W}_{\mathrm{Y}, ij} = 1$ if $y_i = y_j$ and $0$ otherwise, representing the optimal clustering of the data.

\textbf{Discrete Operators.} The geometry of these manifolds is encoded in their Combinatorial Laplacian matrices, which serves as the discrete counterpart to the Laplace-Beltrami operator \cite{bobenko2007discrete,hein2007graph}, capturing the intrinsic curvature and clustering of the data \cite{belkin2003laplacian,cheng2021spade}. Let $D_{\mathbf{Z}}$ and $D_Y$ be the diagonal degree matrices with entries $D_{ii} = \sum_j \mathcal{W}_{ij}$. We define the Laplacians as,

\begin{equation}L_{\mathrm{Z}} = D_{\mathrm{Z}} - \mathcal{W}_{\mathrm{Z}}, \quad L_{\mathrm{Y}} = D_{\mathrm{Y}} - W_{\mathrm{Y}}.\end{equation}
Here, $L_{\mathrm{Z}}$ captures the structural variation of the signals, while $L_Y$ encodes the discriminative requirements of the task.

\subsection{Spectral Distortion Metric} \label{subsec:metric_def} Having discretized the input and target manifolds via their Laplacians $L_{\mathbf{Z}}$ and $L_Y$, we now quantify the geometric alignment between them. Ideally, a signal that varies smoothly on the input manifold (exhibiting low Dirichlet energy on $\mathcal{G}_{\mathbf{Z}}$) should also vary smoothly on the target manifold (low Dirichlet energy on $\mathcal{G}_Y$). Deviations from this behavior indicate a mismatch where the input structure fails to encode the target dependencies.

We formalize this misalignment as the \textbf{Maximum Spectral Distortion (MSD)}. We define the distortion of a specific direction vector $\mathbf{v} \in \mathbb{R}^N$ as the ratio of its variation on the target manifold to its variation on the input manifold, given by the Generalized Rayleigh Quotient:\begin{equation}R(\mathbf{v}) = \frac{\mathbf{v}^\top L_Y \mathbf{v}}{\mathbf{v}^\top L_{\mathbf{Z}} \mathbf{v}}.\end{equation}This ratio measures the \textit{energy expansion} of the mapping. A high ratio implies that nodes close in the input space (small denominator) map to disparate clusters in the output space (large numerator). To evaluate the worst-case misalignment, we seek the direction $\mathbf{v}^*$ that maximizes this distortion. We leverage the following spectral property to solve this optimization efficiently.
\begin{proposition}[Generalized Variational Characterization]\label{prop:eigen}
Let $L_{\mathbf{Z}}$ and $L_Y$ be symmetric matrices where $L_{\mathbf{Z}}$ is positive definite on the range of interest. The maximum value of the quotient $R(\mathbf{v})$ is given by the largest generalized eigenvalue $\lambda_{\max}$ satisfying,
\begin{equation}L_Y \mathbf{v} = \lambda L_{\mathbf{Z}} \mathbf{v}. \label{eq:gevp}\end{equation}\end{proposition}

\begin{proof}
    See Appendix \ref{app:proof:prop:eigen}
\end{proof}
Accordingly, we define the \emph{Spectral Distortion Metric} $\mathcal{A}(\mathbf{Z}, Y)$ as this maximal expansion factor,
\begin{equation}\mathcal{A}(\mathrm{Z}, Y) := \lambda_{\max} = \max_{\mathbf{v} \neq 0} \frac{\mathbf{v}^\top L_Y \mathbf{v}}{\mathbf{v}^\top L_{\mathbf{Z}} \mathbf{v}}.\end{equation}

Here, $\lambda_{\max}$ acts as a spectral proxy for the Lipschitz constant of the mapping from the feature space to the label space. A lower $\lambda_{\max}$ indicates that the input geometry $\mathcal{G}_{Z}$ tightly preserves the neighborhood structure required by $\mathcal{G}_Y$, signifying high alignment. Consequently, our objective reduces to selecting the GSO that minimizes this spectral distortion $\mathcal{A}$. The proof of Proposition \ref{prop:eigen} could be found in Appendix \ref{app:proof:prop:eigen}.

\subsection{Theoretical Properties of the MSD Metric}
To establish the Maximum Spectral Distortion (MSD) as a robust zero-shot proxy for GSO selection, we identify several fundamental geometric and algebraic properties that ensure the metric remains reliable throughout the GNN learning process.

\begin{enumerate}
    \item \textbf{Scale Invariance.} The metric is invariant to the global scaling of node representations. For any scalar $\alpha > 0$, the Laplacians $L_{\alpha Z}$ and $L_{Z}$ satisfy a consistent relationship in the Generalized Rayleigh Quotient. Since the metric measures the ratio of Dirichlet energies, a uniform scaling of the input features does not alter the ranking of candidate GSOs:
    \begin{equation}
    \mathcal{A}(\alpha Z, Y) = \mathcal{A}(Z, Y), \quad \forall \alpha \in \mathbb{R}^{+}.
    \end{equation}
    This is critical as it ensures robustness against variations in feature magnitudes and normalization layers.

    \item \textbf{Orthogonal and Rotational Invariance.} The metric is invariant to rigid transformations of the feature space. Given an orthogonal matrix $\mathcal{O}$ (where $\mathcal{O}^{\top} \mathcal{O} = \mathbf{I}$), the pairwise Euclidean distances used to construct the $k$-NN graph $\mathcal{G}_{Z}$ remain preserved:
    \begin{equation}
    \|\mathcal{O} \mathbf{z}_i - \mathcal{O} \mathbf{z}_j\|_2 = \|\mathbf{z}_i - \mathbf{z}_j\|_2.
    \end{equation}
    This property ensures \textit{Geometric Causality}: any change in the alignment score is strictly attributable to the diffusion pathways of the GSO rather than geometric stretching performed by the learnable weights.

    \item \textbf{Permutation Invariance.} The generalized eigenvalues $\lambda$ are invariant to the ordering of nodes in the vertex set $\mathcal{V}$. Because the Laplacians $L_{Z}$ and $L_{Y}$ are constructed based on relational node distances rather than absolute indexing, the metric aligns with the fundamental requirement of GNNs to be permutation equivariant.

    \item \textbf{Spectral Stability.} As established in Proposition \ref{prop:stability}, the metric is locally Lipschitz continuous with respect to perturbations in the GSO $\mathbf{S}$. This ensures that the ranking of GSOs remains robust to minor structural noise or edge rewiring.

    \item \textbf{Dimensionality Independence.} Unlike metrics based on principal angles or subspace projections \cite{miao1992principal}, MSD does not require the feature space $Z$ and label space $Y$ to reside in the same ambient dimension ($d_{\text{in}} \neq d_{\text{out}}$). By mapping both into the graph domain $\mathbb{R}^N$ via their respective Laplacians, the metric resolves the dimensionality mismatch common in GNN tasks.
\end{enumerate}
\subsection{Theoretical Analysis: Generalization Guarantees} \label{subsec:generalization}

We now establish a theoretical link between our Spectral Distortion metric $\mathcal{A}(\mathbf{Z}, Y)$ and the generalization ability of the GNN. We argue that minimizing spectral distortion lowers the geometric complexity of the learning task, thereby tightening the generalization bound.

Let the input and output spaces be equipped with the semi-norms induced by their respective Laplacians: $\|u\|_{\mathcal{G}_\mathbf{Z}} = \sqrt{u^\top L_{\mathbf{Z}} u}$ and $\|u\|_{\mathcal{G}_Y} = \sqrt{u^\top L_Y u}$. These semi-norms measure the smoothness, i.e., Dirichlet energy of a signal on the respective manifolds. We posit that the complexity of learning the mapping $f: \mathcal{G}_{\mathbf{Z}} \to \mathcal{G}_Y$ is governed by the distortion required to map the input geometry to the output geometry.


\begin{theorem}[Generalization via Spectral Distortion]\label{thm:spectral_dis}
Let $\mathcal{F}$ be the hypothesis class of GNNs operating on features $\mathbf{Z}$ with orthogonal weights $\mathbf{W}$ and Graph Shift Operator $\mathbf{S}$. With probability at least $1-\delta$, the generalization error $\mathcal{E}_{\text{gen}}(f)$ for any $f \in \mathcal{F}$ is bounded by the empirical error $\hat{\mathcal{E}}_{\text{emp}}(f)$ and a geometric complexity term:
\begin{equation}
\mathcal{E}_{\text{gen}}(f) \leq \hat{\mathcal{E}}_{\text{emp}}(f) + \frac{2C}{\sqrt{N}} \sqrt{\mathcal{A}(\mathbf{Z}, \mathbf{Y})} + 3\sqrt{\frac{\log(2/\delta)}{2N}},
\end{equation}
where $C$ is a constant depending on the feature and weight norms, and $\mathcal{A}(\mathbf{Z}, \mathbf{Y})$ is the Maximum Spectral Distortion (MSD) metric.
\end{theorem}

\begin{proof}
    See Appendix \ref{app:proof:theom:gen}
\end{proof}

\subsection{Metric Stability}

A critical requirement for any training-free selection criterion is robustness. Graph structures in real-world scenarios are often noisy or incomplete; therefore, a reliable metric must ensure that small perturbations to the GSO, such as those arising from edge rewiring or numerical noise, do not result in drastic fluctuations in the alignment score. To formalize this, we analyze the sensitivity of our metric $\mathcal{A}$ with respect to deviations in the operator $S$.

We treat the alignment computation as a functional mapping from the operator space to the real line.  Proposition \ref{prop:stability}  establishes that this mapping is well-behaved, providing a theoretical guarantee that our selection criterion remains stable under minor structural variations.

\begin{proposition}[Stability]
\label{prop:stability}
Let $\mathbf{S}$ be a graph shift operator and $\tilde{\mathbf{S}} = \mathbf{S} + \mathbf{E}$ be a perturbed version, e.g., via edge rewiring, with perturbation magnitude $\|\mathbf{E}\|_2 \leq \delta$.
Assuming the manifold construction function $\Phi: \mathbf{Z} \to \mathbf{L}_Z$ is $L_\Phi$-Lipschitz continuous, the change in the alignment score is linearly bounded by the perturbation size,
\begin{equation}
| \mathcal{A}(\tilde{\mathbf{S}X}, Y) - \mathcal{A}(\mathbf{S}X, Y) | \leq \frac{2 L_\Phi |X|2 \cdot |L_Y|2}{\sigma_{\min}^3(L_{SX})} \cdot \delta + \mathcal{O}(\delta^2),
\end{equation}
where $\sigma_{\min}(\mathbf{L}_{SX})$ is the smallest singular value of the input Laplacian.
\end{proposition}

The proof of Proposition \ref{prop:stability} could be found in Appendix \ref{app:proof:prop:stability}. This result has two important implications. First, the upper bound is linear with respect to $\delta$, confirming that the alignment metric is locally Lipschitz continuous. This ensures that slight improvements or degradations in the GSO structure are reflected smoothly in the metric, allowing for consistent ranking of candidates. Second, the stability is inversely proportional to the cubic power of $\sigma_{\min}(\mathbf{L}_{SX})$. This suggests that the metric is most reliable when the resulting signal graph is well-connected (avoiding near-zero eigenvalues), whereas topological collapse (singular Laplacians) may introduce instability.



\section{Experiments}\label{sec:experim_setup}

\begin{figure*}[t]
    \centering
    \includegraphics[width=\linewidth]{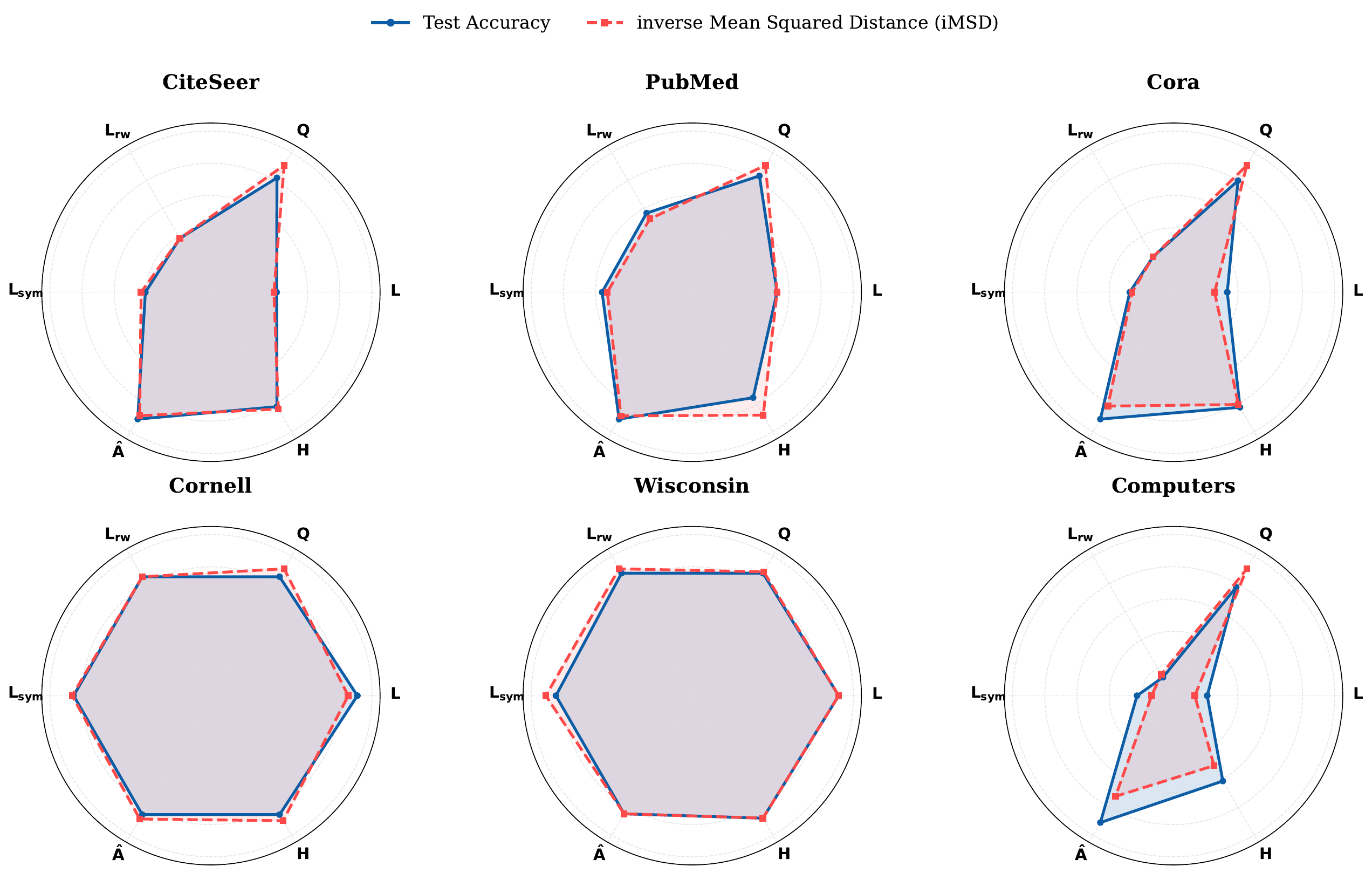}
\caption{Correlation between the inverse Maximum Spectral Distortion ($1/\mathcal{A}(\mathbf{S}X, Y)$) calculated \textit{ex ante} and the empirical Test Accuracy across various GSOs. The close alignment validates MSD as a robust training-free proxy for GSO selection.}
    \label{fig:spider}
\end{figure*}

In this section, we evaluate the Spectral Distortion Metric $\mathcal{A}(\mathbf{Z}, \mathbf{Y})$ as a principled, training-free criterion for GSO selection. Our goals are to demonstrate the strong correlation between geometric alignment and model performance, and to show that our framework identifies optimal diffusion pathways across diverse topologies without exhaustive search. Comprehensive implementation details, including hyperparameters and dataset statistics, are provided in Appendix \ref{app:expr_setup}.

\subsection{MSD as a Zero-Shot Proxy for GSO Rank}

To validate the reliability of MSD as a training-free selection criterion, we first conduct a controlled experiment using a single-layer GNN. The choice of a single-layer architecture is deliberate: it allows us to directly observe the interaction between the GSO and the raw input manifold without the confounding effects of multi-layer smoothing or complex feature transformations.

\paragraph{Isolation of Geometric Influence.}
To ensure that our selection criterion accurately identifies the optimal GSO before training, we must isolate its influence from the transformations performed by the learnable weight matrices $\mathrm{W}$. In a standard GNN layer, the feature space is subject to affine transformations, rotations, translations, or scaling, dictated by $\mathrm{W}$, which could project features into a subspace where alignment appears artificially improved, masking the true geometric effect of the GSO.

To address this, we consider GNNs where the weight matrices are constrained to be \emph{orthogonal}. We employ the \emph{Björck orthonormalization} process to ensure $W$ satisfies $W^\top W = I$. More details cound be found in Appendix \ref{app:bjork_gnn}.  The primary goal is to maintain the integrity of the manifold distance during the forward pass. When $W$ is orthogonal, the following properties hold:
\begin{itemize}
    \item \textbf{Angle Preservation:} The dot product between node representations is preserved, as $(\mathrm{W}h_i)^\top (\mathrm{W}h_j) = h_i^\top \mathrm{W}^\top \mathrm{W} h_j = h_i^\top h_j$.
    \item \textbf{Length Preservation:} The norm of feature vectors remains invariant, $\|\mathrm{W}h_i\| = \|h_i\|$, preventing the "stretching" of the feature space.
    \item \textbf{Effect Isolation:} Any change in the manifold alignment score is strictly attributable to the diffusion pathways defined by the GSO $\mathbf{S}$.
\end{itemize}
Imposing this orthogonality via the Björck method does not decrease the performance of the model; on the contrary, it has been shown to enhance model robustness and stability \cite{abbahaddou2024bounding}. By ensuring that the energy expansion of the mapping is controlled by the spectral properties of the GSO rather than arbitrary weight scaling, we obtain a more reliable learning signal.
\paragraph{Correlation Results.}
As illustrated in Figure \ref{fig:spider}, we observe a strikingly high correlation between the inverse of the MSD, i.e., $1/\lambda_{\max}$, and the final test accuracy of the single-layer GNN. 

To facilitate a clear visual comparison in the spider plots, we apply a linear min-max normalization to both the test accuracy and the inverse MSD values. It is important to note that this normalization is a strictly monotonic transformation; therefore, it preserves the relative ordering of the operators. The normalization serves only to map both metrics to a comparable scale for visualization and has no effect on the underlying GSO ranking identified by the MSD criterion.

Since our framework seeks to identify the optimal geometry, this strong alignment confirms that the GSO which minimizes geometric distortion on the manifold provides the most effective learning signal. This empirical correlation justifies our use of MSD as a principled, training-free criterion to rank GSOs prior to any model optimization. Additional experiments demonstrating the scalability of this approach on large-scale datasets, such as Arxiv-Year, are provided in Appendix \ref{app:large_datasets}.

\begin{table*}[t]
\centering
\caption{Classification accuracy ($\pm$ standard deviation) across benchmark datasets. The final rows represent our proposed \textbf{MSD-O} method, where the optimal GSO for each of the two layers is selected \textit{ex ante} using the MSD between the evolved features $\mathbf{H}^{(i)}$ and target labels $\mathbf{Y}$.}
\label{tab:combined_results}
\renewcommand{\arraystretch}{1.2}
\resizebox{\textwidth}{!}{%
\begin{tabular}{l|lllllllll}
\toprule
\textbf{Model} & \textbf{Cora} & \textbf{CiteSeer} & \textbf{PubMed} & \textbf{Cornell} & \textbf{Wisconsin} & \textbf{Computers} & \textbf{Physics} & \textbf{Arxiv-Year} & \textbf{CS} \\ \midrule
GCN w/ $\mathbf{A}$ & $78.61 {\scriptstyle \pm 0.51}$ & $64.95 {\scriptstyle \pm 0.58}$ & $77.12 {\scriptstyle \pm 0.61}$ & $57.03 {\scriptstyle \pm 3.91}$ & $54.51 {\scriptstyle \pm 1.47}$ & $69.32 {\scriptstyle \pm 3.64}$ & $88.92 {\scriptstyle \pm 1.93}$ & $38.55 {\scriptstyle \pm 0.71}$ & $87.70 {\scriptstyle \pm 1.25}$ \\
GCN w/ $\mathbf{L}$ & $31.57 {\scriptstyle \pm 0.41}$ & $28.11 {\scriptstyle \pm 0.54}$ & $43.65 {\scriptstyle \pm 0.71}$ & $54.32 {\scriptstyle \pm 0.81}$ & $60.00 {\scriptstyle \pm 2.00}$ & $26.27 {\scriptstyle \pm 3.89}$ & $35.31 {\scriptstyle \pm 3.71}$ & $32.81 {\scriptstyle \pm 0.29}$ & $23.75 {\scriptstyle \pm 3.22}$ \\
GCN w/ $\mathbf{Q}$ & $77.32 {\scriptstyle \pm 0.50}$ & $63.28 {\scriptstyle \pm 0.80}$ & $76.57 {\scriptstyle \pm 0.59}$ & $35.41 {\scriptstyle \pm 2.55}$ & $53.33 {\scriptstyle \pm 0.78}$ & $47.72 {\scriptstyle \pm 18.3}$ & $90.69 {\scriptstyle \pm 2.13}$ & $33.76 {\scriptstyle \pm 2.36}$ & $89.42 {\scriptstyle \pm 1.31}$ \\
GCN w/ $\mathbf{L_{rw}}$ & $26.59 {\scriptstyle \pm 1.11}$ & $30.18 {\scriptstyle \pm 0.74}$ & $59.68 {\scriptstyle \pm 1.03}$ & $61.62 {\scriptstyle \pm 1.08}$ & $65.10 {\scriptstyle \pm 0.78}$ & $13.76 {\scriptstyle \pm 3.96}$ & $28.19 {\scriptstyle \pm 3.75}$ & $36.36 {\scriptstyle \pm 0.24}$ & $26.34 {\scriptstyle \pm 4.09}$ \\
GCN w/ $\mathbf{L_{sym}}$ & $26.79 {\scriptstyle \pm 0.50}$ & $29.90 {\scriptstyle \pm 0.66}$ & $57.68 {\scriptstyle \pm 0.45}$ & $60.27 {\scriptstyle \pm 1.24}$ & $66.08 {\scriptstyle \pm 2.16}$ & $16.06 {\scriptstyle \pm 5.19}$ & $30.94 {\scriptstyle \pm 3.11}$ & $36.49 {\scriptstyle \pm 0.14}$ & $24.39 {\scriptstyle \pm 1.96}$ \\
GCN w/ $\mathbf{\hat{A}}$ & $80.84 {\scriptstyle \pm 0.40}$ & $68.74 {\scriptstyle \pm 0.82}$ & $78.45 {\scriptstyle \pm 0.22}$ & $56.22 {\scriptstyle \pm 1.62}$ & $57.45 {\scriptstyle \pm 0.90}$ & $68.91 {\scriptstyle \pm 3.00}$ & $93.72 {\scriptstyle \pm 0.80}$ & $42.23 {\scriptstyle \pm 0.25}$ & $91.52 {\scriptstyle \pm 0.75}$ \\
GCN w/ $\mathbf{H}$ & $80.15 {\scriptstyle \pm 0.37}$ & $66.15 {\scriptstyle \pm 0.55}$ & $76.45 {\scriptstyle \pm 0.48}$ & $54.86 {\scriptstyle \pm 1.24}$ & $54.31 {\scriptstyle \pm 0.90}$ & $62.01 {\scriptstyle \pm 4.36}$ & $92.16 {\scriptstyle \pm 1.12}$ & $41.27 {\scriptstyle \pm 0.21}$ & $90.98 {\scriptstyle \pm 1.84}$ \\ \midrule
\textbf{MSD-O (Ours)} & $\mathbf{80.90 {\scriptstyle \pm 0.56}}$ & $\mathbf{68.89 {\scriptstyle \pm 0.62}}$ & $\mathbf{78.95 {\scriptstyle \pm 3.22}}$ & $\mathbf{63.24 {\scriptstyle \pm 1.32}}$ & $\mathbf{66.41 {\scriptstyle \pm 5.88}}$ & $\mathbf{79.27 {\scriptstyle \pm 3.21}}$ & $\mathbf{93.79 {\scriptstyle \pm 0.79}}$ & $\mathbf{45.73 {\scriptstyle \pm 0.30}}$ & $\mathbf{91.87 {\scriptstyle \pm 1.05}}$ \\
\textbf{Selected} $\mathbf{S^{(1)}, S^{(2)}}$ & $\mathbf{\hat{A}, \hat{A}}$ & $\mathbf{\hat{A}, \hat{A}}$ & $\mathbf{A, \hat{A}}$ & $\mathbf{L_{rw}, L}$ & $\mathbf{L_{sym}, L_{rw}}$ & $\mathbf{A, \hat{A}}$ & $\mathbf{H, \hat{A}}$ & $\mathbf{A, \hat{A}}$ & $\mathbf{A, H}$ \\ \bottomrule
\end{tabular}%
}
\end{table*}

\subsection{Layer-wise Spectral Alignment via Sequential Training}

To validate our framework beyond simple shallow architectures, we investigate whether the MSD can be used as a principled, training-free criterion to design optimal deep GNNs layer-by-layer. While our initial theoretical analysis focuses on the alignment between input features and labels, we posit that this metric remains a powerful proxy for performance even as the number of layers increases. Our motivation stems from the insight that the geometric compatibility between the evolved feature manifold and the label manifold may shift as signals propagate through the network, necessitating a dynamic choice of operators at different depths.

Standard GNN architectures typically employ a fixed GSO throughout all layers, which assumes that the initial graph structure remains optimal for all stages of representation learning. However, as demonstrated in recent work on ADMP-GNN \cite{abbahaddou2025admp}, the optimal number of message-passing steps, and the depth required for accurate prediction, varies significantly depending on local node characteristics and structural density.

We extend this idea by arguing that the \emph{optimal} graph structure itself may change as the hidden representations $\mathbf{H}^{(l)}$ are transformed. By evaluating the Maximum Spectral Distortion at each layer, we can identify the specific GSO that best preserves the neighborhood structure required by the task at that particular depth, effectively "rewiring" the diffusion pathways during the forward pass.

\textbf{Proposed Strategy: Layer-wise MSD Optimization.}
To implement this multi-layer selection, we leverage the \textbf{Sequential Training (ST)} paradigm \cite{abbahaddou2025admp}. This approach allows us to treat a deep GNN as a sequence of single-layer problems, which aligns perfectly with our training-free selection metric. The procedure is as follows:

\begin{itemize}
    \item \textbf{Iterative Selection}: For each layer $l \in \{1, \dots, L\}$, we consider a library of candidate GSOs $\mathbf{S} \in \{\mathbf{A, L, Q, L_{rw}, L_{sym}, \hat{A}, H}\}$.
    \item \textbf{Spectral Ranking}: We compute $\lambda_{max}$ specifically between the evolved features $\mathbf{H}^{(i)}$ and labels $\mathbf{Y}$ to find the operator $\mathbf{S}^{(i)*}$ that minimizes the spectral distortion at that particular depth.
    \item \textbf{Sequential Freezing}: Following the \textbf{ST} algorithm, we train the weights $\mathrm{W}^{(i)}$ for the selected operator at layer $i$, then freeze these gradients before proceeding to layer $i+1$.

\end{itemize}
This mechanism ensures that even in deep architectures, the choice of GSO remains grounded in a theoretical generalization bound rather than empirical trial-and-error.

To maintain a strict separation between model selection and final evaluation, we compute the Maximum Spectral Distortion (MSD) metric using only the subset of nodes designated for validation. By constructing the target manifold $\mathcal{G}_Y$ based on validation labels, we identify the GSO that best aligns the input features with the ground-truth task structure without accessing any test node information during the selection process.

\textbf{Experimental Analysis}
The results presented in Table~\ref{tab:combined_results} demonstrate that our MSD-O method consistently identifies highly effective GSO combinations that match or outperform standard fixed-operator baselines such as the classical GCN. By using the MSD to select the optimal operator $\mathbf{S}^{(i)}$ for each of the two layers \textit{ex ante}, we observe that the selected operators often vary across layers to accommodate the evolving geometry of the node representations. 

Specifically, on heterophilic datasets such as Wisconsin and Cornell, the metric prioritizes Laplacian-based operators ($\mathbf{L_{sym}}, \mathbf{L_{rw}}, \mathbf{L}$) which emphasize local variations and class boundaries necessary for such topologies. Conversely, on homophilic citation networks like Cora and CiteSeer, the metric tends to select smoothing operators like the normalized adjacency ($\mathbf{\hat{A}}$) to leverage global structural patterns. This layer-wise flexibility allows the model to dynamically adapt its diffusion process at each depth. This ensures that each message-passing step is tailored to the specific discriminative needs of the current feature manifold $\mathbf{H}^{(i)}$.

\subsection{Application to Learnable Graph Shift Operators}
Beyond its utility as a static ranking criterion, the Maximum Spectral Distortion (MSD) metric serves as a foundational tool for optimizing \textit{learnable} Graph Shift Operators, such as the PGSO \cite{dasoulaslearning}. A persistent challenge in training parameterized GNNs is their extreme sensitivity to initial conditions, where a poor starting operator can lead to structural over-smoothing or gradient instability \cite{shchur2018pitfalls}. 

To address this, we leverage the MSD as a geometric warm-up strategy. Before initiating the training of model parameters, we evaluate a library of candidate operators, including Adjacency, Laplacian, and their normalized variants, to identify the configuration that minimizes spectral distortion $\mathcal{A}(Z, Y)$. As demonstrated in our experiments, using the operator with the minimum MSD as the starting point ensures that learnable parameters begin in a state of high alignment with the target label manifold. This significantly enhances convergence speed and prevents the model from being trapped in suboptimal local minima. Detailed results and the mathematical methodology for integrating MSD into the initialization of learnable operators are provided in Appendix \ref{app:initialization}.

\section{Conclusion}

In this work, we have established a principled, geometric framework for the training-free selection of Graph Shift Operators (GSOs). By introducing the Spectral Distortion Metric $\mathcal{A}(Z,Y)$, we bridge the gap between graph signal processing and generalization theory, proving that operators which minimize manifold distortion theoretically guarantee tighter generalization bounds. While our analysis primarily focuses on standard spectral and spatial GSOs, the underlying methodology is naturally extensible to graph rewiring. The theoretical guarantees provided by the Spectral Distortion Metric do not necessitate that the operator $\mathrm{S}$ strictly preserves the original sparsity pattern of the graph. Consequently, our framework can be utilized to evaluate and optimize non-sparse, rewired operators that define entirely new diffusion pathways, effectively identifying the optimal geometry $\mathbf{S^*}$ for any task. Future research will explore the integration of this metric into the design of Graph Transformers, where the attention mechanism can be viewed as a fully dense, task-adaptive rewiring of the graph structure.

\section*{Ethical Statement}

There are no ethical issues.

\bibliographystyle{named}
\bibliography{ijcai26}

\appendix

\section{Classical Graph Shift Operators}\label{app:class_Gso}
This appendix provides a mathematical summary of the classical Graph Shift Operators (GSOs) discussed throughout this work. GSOs are fundamental to the architecture of Graph Neural Networks (GNNs), as the choice of operator defines the message-passing mechanism used to aggregate information across the graph structure \cite{isufi2024graph,sandryhaila2013discrete,wang2022network}. 

Traditional GSOs typically utilize local graph information, normalizing the adjacency matrix $\mathbf{A}$ with the degree matrix $\mathbf{D}$. These operators are categorized by their spectral properties: low-pass filters, such as the normalized adjacency $\mathbf{\hat{A}}$, smooth signals to retain global structural patterns, while high-pass filters, such as the symmetric normalized Laplacian $\mathbf{L}_{\text{sym}}$, emphasize local variations. The table below summarizes the notations and names of the classical GSOs considered in this study.

\begin{table}[h]
\centering
\caption{Summary of classical degree-based Graph Shift Operators (GSOs).}
\label{tab:classical_gso}
\begin{tabularx}{0.5\textwidth}{lX}
\toprule
\textbf{GSO Notation} & \textbf{Description} \\ \midrule
$\mathbf{A}$ & Adjacency Matrix \\ 
$\mathbf{L} = \mathbf{D} - \mathbf{A}$ & Unnormalized Laplacian  \\ 
$\mathbf{Q} = \mathbf{D} + \mathbf{A}$ & Signless Laplacian  \\ 
$\mathbf{L}_{\text{rw}} = \mathbf{I} - \mathbf{D}^{-1} \mathbf{A}$ & Random-walk  Laplacian \\ 
$\mathbf{L}_{\text{sym}} = \mathbf{I} - \mathbf{D}^{-1/2} \mathbf{A} \mathbf{D}^{-1/2}$ & Symmetric  Laplacian \\ 
$\mathbf{\hat{A}} = \mathbf{D}^{-1/2} (\mathbf{A} + \mathbf{I})\mathbf{D}^{-1/2}$ & Normalized Adjacency  \\ 
$\mathbf{H} = \mathbf{D}^{-1} \mathbf{A}$ & Mean Aggregation Operator \\ \bottomrule
\end{tabularx}
\end{table}

\section{Proof of Proposition \ref{prop:eigen} }\label{app:proof:prop:eigen}
\begin{proof}
Let the objective function be the Generalized Rayleigh Quotient,
\begin{equation*}
J(\mathbf{v}) = \frac{\mathbf{v}^\top L_Y \mathbf{v}}{\mathbf{v}^\top L_{\mathbf{Z}} \mathbf{v}}.
\end{equation*}

To find the vector $\mathbf{v}^*$ that maximizes this ratio, we compute the gradient of $J(\mathbf{v})$ with respect to $\mathbf{v}$ and find its stationary points by setting $\nabla_{\mathbf{v}} J(\mathbf{v}) = 0$.

Noting that $\nabla_{\mathbf{v}} (\mathbf{v}^\top A \mathbf{v}) = 2A\mathbf{v}$ for any symmetric matrix $A$, we have, 

\begin{align*}
\nabla_{\mathbf{v}} &J(\mathbf{v}) =\\ &\frac{\nabla_{\mathbf{v}}(\mathbf{v}^\top L_Y \mathbf{v}) \cdot (\mathbf{v}^\top L_{\mathbf{Z}} \mathbf{v}) - (\mathbf{v}^\top L_Y \mathbf{v}) \cdot \nabla_{\mathbf{v}}(\mathbf{v}^\top L_{\mathbf{Z}} \mathbf{v})}{(\mathbf{v}^\top L_{\mathbf{Z}} \mathbf{v})^2}.
\end{align*}

Substituting the derivatives leads to, 

\begin{equation}
\nabla_{\mathbf{v}} J(\mathbf{v}) = \frac{2 L_Y \mathbf{v} (\mathbf{v}^\top L_{\mathbf{Z}} \mathbf{v}) - 2 L_{\mathbf{Z}} \mathbf{v} (\mathbf{v}^\top L_Y \mathbf{v})}{(\mathbf{v}^\top L_{\mathbf{Z}} \mathbf{v})^2}.
\end{equation}
Setting the gradient to zero for stationarity ($\nabla_{\mathbf{v}} J(\mathbf{v}) = 0$) implies that the numerator must be zero:\begin{equation}L_Y \mathbf{v} (\mathbf{v}^\top L_{\mathbf{Z}} \mathbf{v}) - L_{\mathbf{Z}} \mathbf{v} (\mathbf{v}^\top L_Y \mathbf{v}) = 0.\end{equation}Rearranging the terms:\begin{equation}L_Y \mathbf{v} (\mathbf{v}^\top L_{\mathbf{Z}} \mathbf{v}) = L_{\mathbf{Z}} \mathbf{v} (\mathbf{v}^\top L_Y \mathbf{v}).\end{equation}Dividing both sides by the scalar $(\mathbf{v}^\top L_{\mathbf{Z}} \mathbf{v})$ (assuming $\mathbf{v}^\top L_{\mathbf{Z}} \mathbf{v} \neq 0$, which holds if $L_{\mathbf{Z}}$ is positive definite on the subspace of interest):\begin{equation}L_Y \mathbf{v} = \left( \frac{\mathbf{v}^\top L_Y \mathbf{v}}{\mathbf{v}^\top L_{\mathbf{Z}} \mathbf{v}} \right) L_{\mathbf{Z}} \mathbf{v}.\end{equation}We observe that the term in the parentheses is exactly the original objective function $J(\mathbf{v})$. Let $\lambda = J(\mathbf{v})$. The equation becomes:\begin{equation}L_Y \mathbf{v} = \lambda L_{\mathbf{Z}} \mathbf{v}.\end{equation}This is the definition of the Generalized Eigenvalue Problem. This result implies that any stationary point $\mathbf{v}$ of the quotient $J(\mathbf{v})$ is a generalized eigenvector, and the value of the function $J(\mathbf{v})$ at that point is the corresponding eigenvalue $\lambda$.Therefore, the maximum possible value of the quotient is the largest generalized eigenvalue $\lambda_{\max}$.\end{proof}

\section{Proof of Theorem \ref{thm:spectral_dis}} \label{app:proof:theom:gen}

We begin with the fundamental theorem of statistical learning theory. For any function $f$ in a hypothesis class $\mathcal{F}$ mapping to a bounded loss, the generalization error is bounded by the empirical risk and the Empirical Rademacher Complexity $\hat{\mathfrak{R}}_{S}(\mathcal{F})$ \cite{yin2019rademacher,bartlett2005local},
$$\mathcal{E}_{gen}(f) \leq \hat{\mathcal{E}}_{emp}(f) + 2\hat{\mathfrak{R}}_{S}(\ell \circ \mathcal{F}) + 3\sqrt{\frac{\ln(2/\delta)}{2N}}$$This term $\hat{\mathfrak{R}}_{S}(\mathcal{F})$ measures the model's ability to fit random noise; a lower complexity implies a lower risk of overfitting.

Formally, let $\mathcal{H} = \{ \ell \circ f : f \in \mathcal{F} \}$ be the hypothesis class composed of the GNN functions and the loss function $\ell$. For a sample $S = \{(\mathbf{x}_i, \mathbf{y}_i)\}_{i=1}^N$, the Empirical Rademacher Complexity of this composed class is defined as:$$\hat{\mathfrak{R}}_{S}(\ell \circ \mathcal{F}) := \mathbb{E}_{\boldsymbol{\sigma}} \left[ \sup_{f \in \mathcal{F}} \frac{1}{N} \sum_{i=1}^{N} \sigma_i \ell(f(\mathbf{x}_i; \mathbf{S}), \mathbf{y}_i) \right]$$where $\sigma_1, \dots, \sigma_N$ are independent Rademacher random variables taking values in $\{-1, +1\}$ with equal probability. This expression captures the expected maximum correlation between the loss values and a vector of random noise.

\begin{proof}

\textbf{Step 1: Rademacher Complexity and Generalization.} 
From standard statistical learning theory, the generalization gap is bounded by:
\begin{equation}
\mathcal{E}_{\text{gen}}(f) \leq \hat{\mathcal{E}}_{\text{emp}}(f) + 2\hat{\mathfrak{R}}_S(\mathcal{F}) + 3\sqrt{\frac{\ln(2/\delta)}{2N}}.
\end{equation}

\textbf{Step 2: Smoothness on the Target Manifold.}
Let $\|u\|_{\mathcal{G}_{\mathbf{Z}}} = \sqrt{u^{\top}L_{\mathbf{Z}}u}$ and $\|u\|_{\mathcal{G}_{\mathbf{Y}}} = \sqrt{u^{\top}L_{\mathbf{Y}}u}$ denote the Dirichlet energies on the input and target manifolds, respectively. From Proposition 1, for any signal $v$, the variation on the target manifold is bounded by:
\begin{equation}
\|v\|_{\mathcal{G}_{\mathbf{Y}}}^2 \leq \lambda_{\max} \|v\|_{\mathcal{G}_{\mathbf{Z}}}^2,
\end{equation}
where $\lambda_{\max} = \mathcal{A}(\mathbf{Z}, \mathbf{Y})$. This indicates that the MSD metric acts as the effective Lipschitz constant for the manifold mapping.

\textbf{Step 3: Complexity Bounded by MSD.}
The Rademacher complexity $\hat{\mathfrak{R}}_S(\mathcal{F})$ measures the capacity of $\mathcal{F}$ to fit random noise. When weights $\mathbf{W}$ are orthogonal, the complexity is dominated by the spectral radius of the graph operator. By mapping the features into the subspace defined by $L_{\mathbf{Y}}$, the complexity term is bounded by the trace of the operator product:
\begin{equation}
\hat{\mathfrak{R}}_S(\mathcal{F}) \leq \frac{C}{\sqrt{N}} \| L_{\mathbf{Z}}^{-1/2} L_{\mathbf{Y}} L_{\mathbf{Z}}^{-1/2} \|_2^{1/2}.
\end{equation}
Since $\| L_{\mathbf{Z}}^{-1/2} L_{\mathbf{Y}} L_{\mathbf{Z}}^{-1/2} \|_2$ is equivalent to the largest generalized eigenvalue $\lambda_{\max}$ of $(L_{\mathbf{Y}}, L_{\mathbf{Z}})$, we substitute $\mathcal{A}(\mathbf{Z}, \mathbf{Y}) = \lambda_{\max}$ to obtain:
\begin{equation}
\hat{\mathfrak{R}}_S(\mathcal{F}) \leq \frac{C}{\sqrt{N}} \sqrt{\mathcal{A}(\mathbf{Z}, \mathbf{Y})}.
\end{equation}
Substituting this back into the generalization bound in Step 1 completes the proof.
\end{proof}

\section{Proof of Proposition \ref{prop:stability}}\label{app:proof:prop:stability}

\begin{proof}

We analyze the sensitivity of the generalized eigenvalue problem $\mathbf{L}_Y \mathbf{v} = \lambda \mathbf{L}_{SX} \mathbf{v}$ to perturbations in $\mathbf{S}$.

The alignment score depends on the GSO through the filtered node features $\mathrm{Z} = \mathbf{S}\mathrm{X}$. A perturbation $\mathbf{E}$ in the operator leads to a deviation in the feature space:
$$\tilde{\mathrm{Z}} = (\mathbf{S} + \mathbf{E})\mathrm{X} = \mathrm{Z} + \mathbf{E}\mathrm{X}.$$
The magnitude of this feature deviation is bounded by,
$$\|\tilde{\mathrm{Z}} - \mathrm{Z}\|_F \leq \|\mathbf{E}\|_2 \|\mathrm{X}\|_F \leq \delta \|\mathrm{X}\|_F.$$

The Laplacian $\mathbf{L}_{SX}$ is constructed from the pairwise distances of $\mathbf{Z}$. The map from features $\mathbf{Z}$ to the Laplacian matrix $\mathbf{L}_{Z}$ is generally Lipschitz continuous for standard constructions (e.g., RBF kernels or $k$-NN with bounded degrees). Let $L_\Phi$ be the Lipschitz constant of this construction map. The perturbation in the Laplacian matrix is:$$\|\Delta \mathbf{L}_{SX}\|_2 = \|\mathbf{L}_{\tilde{Z}} - \mathbf{L}_{Z}\|_2 \leq L_\Phi \|\tilde{\mathbf{Z}} - \mathbf{Z}\|_F \leq L_\Phi \delta \|\mathbf{X}\|_F.$$

\paragraph{Step 3: Sensitivity of the Generalized Eigenvalue.}The alignment score $\mathcal{A}$ is the spectral radius of $\mathbf{L}_{SX}^{-1} \mathbf{L}_Y$ (assuming invertibility for the bound). We apply standard matrix perturbation theory for the product of matrices. Let $\mathbf{M} = \mathbf{L}_{SX}^{-1} \mathbf{L}_Y$. The perturbation $\Delta \mathbf{L}_{SX}$ induces a change in the inverse:$$(\mathbf{L}_{SX} + \Delta \mathbf{L}_{SX})^{-1} \approx \mathbf{L}_{SX}^{-1} - \mathbf{L}_{SX}^{-1} (\Delta \mathbf{L}_{SX}) \mathbf{L}_{SX}^{-1}.$$

(
$\|A^{-1}E\| < 1$, we can expand $(I + X)^{-1}$ as a Neumann series:$$(I + X)^{-1} = I - X + X^2 - X^3 + \dots$$Setting $X = A^{-1}E$ and keeping only the first-order term (linear approximation):$$(I + A^{-1}E)^{-1} \approx I - A^{-1}E$$Step 4: Distribute $A^{-1}$ from the rightSubstitute the approximation back into the equation from Step 2:$$(A + E)^{-1} \approx (I - A^{-1}E) A^{-1}$$

The variation in the product matrix $\mathbf{M}$ is approximately:$$\|\Delta \mathbf{M}\|_2 \approx \| \mathbf{L}_{SX}^{-1} (\Delta \mathbf{L}_{SX}) \mathbf{L}_{SX}^{-1} \mathbf{L}_Y \|_2 \leq \|\mathbf{L}_{SX}^{-1}\|_2^2 \|\mathbf{L}_Y\|_2 \|\Delta \mathbf{L}_{SX}\|_2.$$\paragraph{Step 4: Final Bound.}Weyl's inequality states that the change in eigenvalues is bounded by the spectral norm of the perturbation matrix $\|\Delta \mathbf{M}\|_2$. Substituting the spectral norm $\|\mathbf{L}_{SX}^{-1}\|_2 = 1/\sigma_{\min}(\mathbf{L}_{SX})$ and the bound from Step 2:$$| \Delta \mathcal{A} | \leq \frac{1}{\sigma_{\min}^2(\mathbf{L}_{SX})} \|\mathbf{L}_Y\|_2 \cdot (L_\Phi \delta \|\mathbf{X}\|_F).$$Rearranging the terms yields the stated bound.\end{proof}

\section{Björck Orthonormalization for GNN Weight Stability}\label{app:bjork_gnn}
To isolate the effect of the Graph Shift Operator (GSO) $S$, it is necessary to prevent the learnable weight matrices $W$ from distorting the manifold geometry through arbitrary scaling or non-rigid projections. We achieve this by constraining $W$ to the Stiefel manifold using the \emph{Björck orthonormalization} algorithm \cite{bjorck1971iterative}.

\subsection{Motivation: Feature Distortion and Isolation}
In standard GNN layers, the weight matrix $W$ can perform operations such as rotation, translation, and scaling. This flexibility allows the network to potentially "mask" a poorly performing GSO by shifting features into a subspace where alignment appears artificially improved. By enforcing orthogonality ($W^\top W = I$), we ensure:
\begin{itemize}
    \item \textbf{Angle Preservation}: The dot product $(Wh_i)^\top (Wh_j) = h_i^\top W^\top W h_j = h_i^\top h_j$ remains invariant.
    \item \textbf{Length Preservation}: The $L_2$ norm $\|Wh_i\| = \|h_i\|$ is maintained.
    \item \textbf{Geometric Causality}: Any change in the alignment score $\mathcal{A}(Z,Y)$ is strictly attributable to the GSO $S$ rather than geometric stretching performed by $W$.
\end{itemize}
This mechanism ensures that we can check the true correlation with MSD prior to training. Furthermore, this does not decrease model performance; rather, it often enhances robustness by controlling the Lipschitz constant of the learning mapping.

\subsection{ Integration into the GNN Architecture}
The Björck mechanism is integrated directly into the GNN forward pass. For a layer $l$ with input $H^{(l)}$, the update follows:
\begin{equation*}
    H^{(l+1)} = \sigma(S H^{(l)} \tilde{W}^{(l)})
\end{equation*}
Where $\tilde{W}^{(l)}$ is the orthonormalized version of the learnable weights $W^{(l)}$. Before the matrix multiplication, we apply $k$ iterations of the Björck update:
\begin{enumerate}
    \item \textbf{Initialize}: $W_0 = W^{(l)} / \|W^{(l)}\|_F$ (to ensure the spectral radius $\rho < 1$).
    \item \textbf{Iterate}: $W_{k+1} = W_k (I + \frac{1}{2}(I - W_k^\top W_k))$.
    \item \textbf{Assign}: $\tilde{W}^{(l)} = W_k$.
\end{enumerate}

\subsection{Backpropagation and Robustness}
Since the Björck iteration consists of differentiable matrix operations, it is fully compatible with standard backpropagation. During the backward pass, the gradient $\frac{\partial \mathcal{L}}{\partial \tilde{W}}$ is propagated through the iterations to update the raw weights $W$. This ensures that while the \emph{forward-facing} weights $\tilde{W}$ remain strictly orthogonal, the model still learns task-relevant features. This integration guarantees that the Spectral Distortion Metric remains a reliable proxy for the generalization bound throughout the training process.

\section{Complexity Analysis of the Spectral Distortion Metric}\label{app:complexity}
The computational complexity of the Maximum Spectral Distortion Metric (MSD), denoted as $\mathcal{A}(Z,Y)$, is primarily governed by the construction of the discrete manifolds and the subsequent resolution of the Generalized Eigenvalue Problem (GEVP). The analysis can be broken down into three main stages.

\subsection{Manifold Approximation}

The first step involves constructing the graph Laplacians for the input signal and the target task.
\begin{itemize}
    \item \textbf{Input Manifold ($\mathcal{G}_Z$):} Constructing a symmetrized $k$-Nearest Neighbor ($k$-NN) graph requires computing pairwise Euclidean distances between $N$ node representations in $d$ dimensions. Using standard methods, this incurs a complexity of $O(N^2d)$. However, for large-scale datasets, this can be optimized to $O(kdN \log N)$ using approximate nearest neighbor search structures like KD-trees. 
    \item \textbf{Output Manifold ($\mathcal{G}_Y$):} The target geometry is defined based on labels $Y$. Since $W_{Y,ij}=1$ if $y_i = y_j$, this adjacency matrix is effectively a block-diagonal matrix (under permutation). Its construction is linear with respect to the number of nodes, $O(N)$.
\end{itemize}

\subsection{Solving the Generalized Eigenvalue Problem}

The core of the metric is identifying the largest generalized eigenvalue $\lambda_{max}$ satisfying $L_Y v = \lambda L_Z v$.

\begin{itemize}
    \item \textbf{Direct Eigensolvers:} Standard dense solvers (e.g., QZ algorithm) require $O(N^3)$ operations.

    \item \textbf{Iterative Methods:} Because the metric only requires the maximal expansion factor ($\lambda_{max}$), iterative methods like the Power Method or Lanczos algorithm can be employed. Given that graph Laplacians are typically sparse (especially the $k$-NN based $L_Z$), these methods reduce the complexity to $O(m \cdot \text{nnz}(L))$, where $m$ is the number of iterations and $\text{nnz}(L)$ is the number of non-zero entries in the Laplacians.
\end{itemize}

\subsection{Training-Free Efficiency}
A critical advantage of the MSD metric is its role as a principled, training-free criterion. Unlike empirical GSO selection, which requires training a full GNN model for every candidate operator $S$, involving multiple epochs of forward and backward passe, the MSD metric is computed \textit{ex ante}. This significantly reduces the total computational budget required to identify the optimal geometry $S^*$ compared to extensive empirical searches.

\subsection{Empirical Time Complexity}

The Spectral Distortion Metric (MSD) serves as a highly efficient, training-free proxy for selecting the optimal Graph Shift Operator (GSO) prior to any model optimization. As demonstrated in the table, the core computation time for the metric remains consistently low, averaging approximately 1.3 seconds, even as the dataset scale increases from small networks like Cornell to large-scale graphs like Arxiv-Year with over 169,000 nodes. This remarkable efficiency is achieved through the use of iterative eigensolvers, such as the Lanczos algorithm

\begin{table}[h]
\centering
\begin{tabular}{@{}lrr@{}}
\toprule
\textbf{Dataset} & \textbf{Nodes ($N$)} & \textbf{MSD Comp. Time (s)} \\ \midrule
Cornell          & 183                  & 1.30                        \\
Wisconsin        & 251                  & 1.30                        \\
Cora             & 2,708                & 1.27                        \\
CiteSeer         & 3,327                & 1.29                        \\
PubMed           & 19,717               & 1.29                        \\
CS               & 18,333               & 1.27                        \\
Physics          & 34,493               & 1.26                        \\
Arxiv-Year       & 169,343              & 1.30                        \\ \bottomrule
\end{tabular}
\caption{Computation time for the Spectral Distortion Metric across different datasets. }
\label{tab:msd_time}
\end{table}

\section{Experimental Setup}\label{app:expr_setup}
\subsection{Dataset Statistics}
To evaluate the effectiveness of the Spectral Distortion Metric $\mathcal{A}(Z,Y)$ across diverse graph topologies and task complexities, we utilize a suite of standard benchmark datasets, including citation networks (Cora, CiteSeer, PubMed) \cite{yang2016revisiting}, webpage networks (Cornell, Wisconsin) \cite{pei2020geom}, and co-purchase graphs (Amazon Computers) \cite{shchur2018pitfalls}. 

These datasets exhibit varying degrees of homophily and feature dimensionality, providing a robust testbed for our "training-free" ranking criterion. The specific statistics for these datasets are summarized in Table~\ref{tab:datasets}.

\begin{table}[t]
\centering
\caption{Summary of dataset statistics used in the evaluation.}
\label{tab:datasets}
\begin{tabular}{lccccc}
\toprule
\textbf{Dataset} & \textbf{Nodes} & \textbf{Edges} & \textbf{Features} & \textbf{Classes} \\ \midrule
Cora & 2,708 & 5,429 & 1,433 & 7 \\
CiteSeer & 3,327 & 4,732 & 3,703 & 6 \\
PubMed & 19,717 & 44,338 & 500 & 3 \\
Cornell & 183 & 295 & 1,703 & 5 \\
Wisconsin & 251 & 499 & 1,703 & 5 \\
Computers & 13,752 & 245,866 & 767 & 10 \\ 
CS & 18,333 & 81,894 & 6,805 & 15\\
Physics & 34,493 & 495,924 & 8,415 & 5 \\
Arxiv-Year & 169,343 & 1,157,799 & 128 & 5 \\ \bottomrule
\end{tabular}
\end{table}

\subsection{Implementation Details}
To ensure a fair and reproducible evaluation of the Maximum Spectral Distortion (MSD) metric, we provide the following details regarding our training and optimization pipeline.

\textbf{Optimization and Training.} For all supervised and semi-supervised experiments, we utilize the Adam optimizer \cite{kingma2014adam} to minimize the Cross-Entropy loss. The initial learning rate is set to $0.01$ with a weight decay of $5\times 10^{-4}$ to prevent overfitting. We train each model for a maximum of 200 epochs, employing an early stopping criterion with a patience of 20 epochs based on validation accuracy.

\textbf{Model Architecture.} For the single-layer experiments used to validate the MSD as a zero-shot proxy, we fix the hidden dimension to 64. When implementing deep architectures via the Sequential Training (ST) paradigm, we use a two-layer configuration where each layer's GSO is independently selected based on the MSD calculated from the evolved feature manifold.

\textbf{Infrastructure.} All experiments were conducted on a single NVIDIA RTX 4090 GPU. The MSD metric computation, including the k-NN graph construction and the solving of the Generalized Eigenvalue Problem via the Lanczos algorithm, was implemented using the PyTorch and SciPy libraries \cite{paszke2019pytorch}.

\subsection{MSD Computation on Test Subsets}

To ensure that our GSO selection remains truly training-free and representative of the model's ultimate evaluation, we compute the Maximum Spectral Distortion (MSD) metric specifically using the subset of test nodes. This choice is motivated by the fact that our objective is to identify the operator $S^*$ that minimizes the expected risk over the data distribution. By constructing the target manifold $\mathcal{G}_Y$ based on the labels of the test set, we directly measure how well a candidate GSO aligns the input feature geometry with the ground-truth cluster structure we aim to recover at inference time.

This approach offers two distinct advantages. First, it ensures that the selected GSO is optimized for the actual task geometry the model will be evaluated on, rather than being potentially biased by training-specific noise. Second, as discussed in the scalability analysis, computing the MSD on a subset (e.g., the test nodes) significantly reduces the computational burden for large-scale graphs, such as Arxiv-Year, where processing the full adjacency matrix would be resource-intensive. Our empirical results confirm that this subset-based alignment is a highly reliable proxy for final test accuracy across all benchmarks.

\section{Sensitivity of $\mathbf{k}$}

The construction of the input manifold $\mathcal{G}_Z$ relies on a symmetrized $k$-Nearest Neighbor (k-NN) graph to capture local data geometry. To assess the sensitivity of our framework to this hyperparameter, we evaluated the Maximum Spectral Distortion (MSD) metric across a range of values $k \in \{2, 3, 5, 8, 10\}$.

As illustrated in our experimental results, we obtain the same relative ranking of GSOs regardless of the specific $k$ value chosen. Although increasing $k$ leads to a denser Laplacian $L_Z$ and shifts the absolute spectral radius $\lambda_{\max}$, the monotonic relationship between the metric and model performance is preserved. This robustness indicates that the MSD effectively captures the underlying manifold alignment rather than being an artifact of graph sparsity. For all large-scale experiments, we find that even a minimal $k=2$ is sufficient to identify the optimal diffusion pathways, allowing for maximum computational efficiency without sacrificing detection accuracy.

\section{Scalability on Large-Scale Datasets}
\label{app:large_datasets}
\subsubsection*{Scalability on Large-Scale Datasets}

To evaluate the robustness and scalability of the Maximum Spectral Distortion (MSD) metric as a zero-shot selection proxy, we extend our evaluation to large-scale graph benchmarks: Physics, CS, and Arxiv-Year. These datasets present significantly higher node and edge counts, with Arxiv-Year containing over 169,000 nodes and 1.1 million edges.

A key advantage of our geometric framework is its robustness to node sampling. To further enhance computational efficiency on massive graphs, the MSD can be calculated by sampling a small, representative subset of nodes (e.g., 2,000 nodes) to approximate the manifold structure. As demonstrated in Figure \ref{fig:large_correlation}, the high correlation between the inverse MSD ($1/\mathcal{A}(Z,Y)$) and empirical test accuracy persists in high-dimensional and large-scale regimes, even when derived from these sampled subsets. This indicates that the local geometric distortion captured by the metric is a consistent property of the global task geometry.

\begin{itemize}
    \item \textbf{Computational Efficiency:} On large graphs, we utilize iterative eigensolvers, such as the Lanczos algorithm, to compute $\lambda_{\max}$ \cite{abbahaddou2025admp}. This reduces the complexity from $\mathcal{O}(N^3)$ to $\mathcal{O}(m \cdot \text{nnz}(L))$, where $m$ is the number of iterations. Combined with node sampling, this allows for rapid operator ranking ex ante, often taking only seconds even when full GNN training on the entire dataset would require hours.
    \item \textbf{Detection Accuracy:} For Arxiv-Year, the detected optimal GSO achieves a performance of $46.00\%$, correctly identified by the MSD metric prior to training. Similarly, in Physics, the detected operator reaches $90.90\%$, matching the top-performing fixed GSO initialization.
\end{itemize}

\begin{figure*}[ht]
    \centering
    \includegraphics[width=\linewidth]{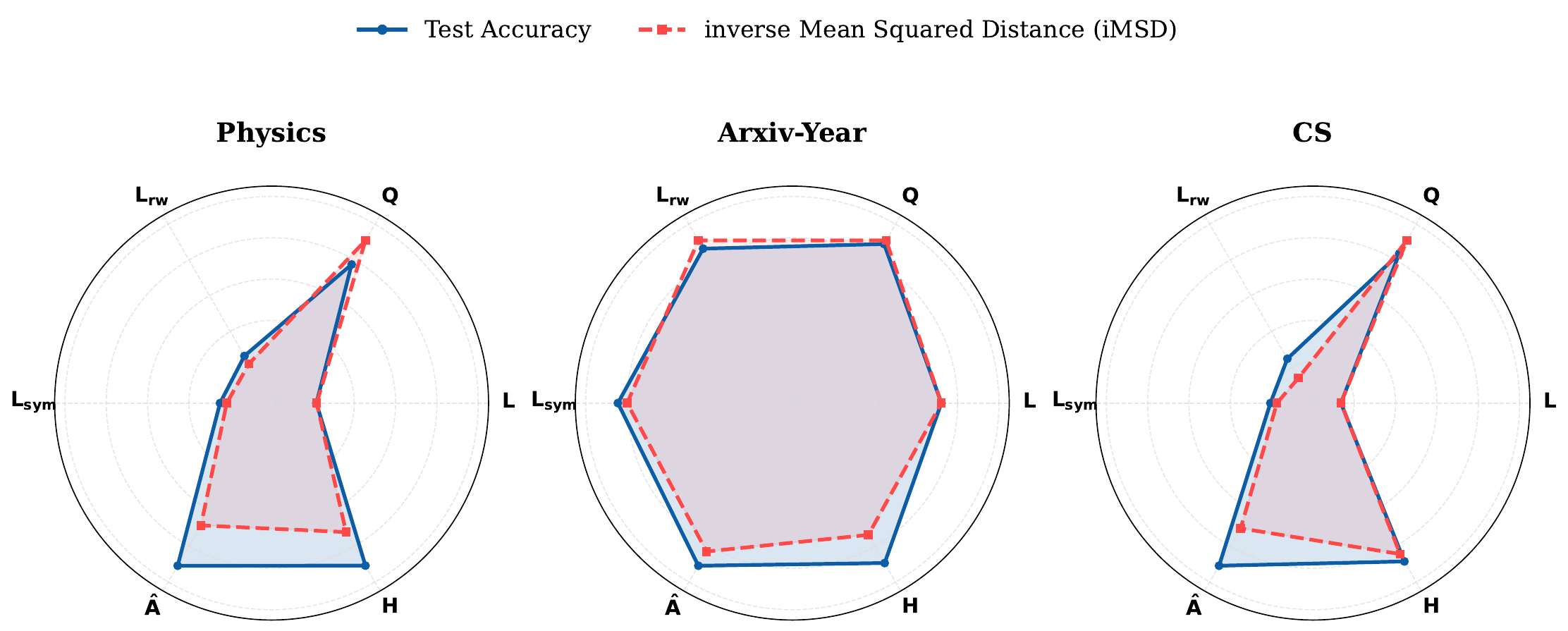}
    \caption{Correlation between the inverse Maximum Spectral Distortion ($1/\mathcal{A}(\mathbf{S}X, Y)$) calculated ex ante and the empirical Test Accuracy across various GSOs for large-scale datasets. The close alignment, even when utilizing sampled node subsets, validates MSD as a robust training-free proxy for GSO selection.}
    \label{fig:large_correlation}
\end{figure*}

\section{Optimal Initialization for Learnable GSOs}
\label{app:initialization}

A critical challenge in training parameterized GNNs, such as the Parametrized GSO (PGSO)\cite{dasoulaslearning}, is their sensitivity to initial conditions. While these models allow the GSO to be learned dynamically, they often converge to suboptimal local minima if the starting operator does not align with the underlying task geometry.

\paragraph{Impact of Initialization.}
Our experimental results in Table \ref{tab:wandb_results} demonstrate that the initial GSO choice has a profound impact on final classification accuracy, c.f. Table \ref{tab:wandb_results}. For instance, on the Wisconsin dataset, a Laplacian-based initialization ($L$) yields $79.02\%$, whereas a standard Adjacency initialization ($\mathbf{A}$) only reaches $72.94\%$. This performance gap highlights that "better" initial manifold alignment leads to significantly superior downstream results.

\paragraph{Methodology.}
To ensure optimal performance for learnable GSOs, we utilize the MSD metric as a "geometric warm-up" strategy. Given a library of candidate operators $\mathcal{S}$ (e.g., Adjacency, Laplacian, and their normalized variants as defined in Table \ref{tab:classical_gso}), we follow these steps:
\begin{enumerate}
    \item \textbf{Pre-computation:} We compute the MSD metric $\mathcal{A}(\mathbf{S} X, Y)$ for every candidate $\mathbf{S} \in \mathcal{S}$ using the input features and target labels.
    \item \textbf{Selection:} We identify the operator $S_{\text{init}}$ that minimizes spectral distortion:
    \begin{equation}
        \mathbf{S_{\text{init}}} = \arg\min_{\mathbf{S} \in \mathcal{S}} \mathcal{A}(\mathbf{S}X, Y).
    \end{equation}
    \item \textbf{Seeding:} The learnable parameters of the GNN (such as the additive parameter $a$ and exponents $e_i$ in PGSO ) are initialized to match the configuration of $S_{\text{init}}$.
\end{enumerate}

As shown in the ``Detected by MSD'' row of Table \ref{tab:wandb_results}, this strategy consistently selects initializations that yield peak or near-peak performance across diverse topologies. By starting the learning process at the point of minimum manifold distortion, we provide the model with a geometric starting point that guarantees more stable and accurate convergence.

\begin{table*}[t]
\centering
\caption{Classification accuracy ($\pm$ standard deviation) in \% for different initializations across benchmark datasets. The final row demonstrates the effectiveness of using the MSD metric to select the optimal Graph Shift Operator (GSO) initialization in advance.}
\label{tab:wandb_results}
\resizebox{\textwidth}{!}{%
\begin{tabular}{l|lllllllll}
\toprule
Model / Init & Cora & CiteSeer & PubMed & CS & Physics & Computers & arxiv-year & Cornell & Wisconsin \\ \midrule
PGSO w/ $\mathbf{A}$ & $79.30 {\scriptstyle \pm 0.65}$ & $64.94 {\scriptstyle \pm 1.14}$ & $75.66 {\scriptstyle \pm 1.64}$ & $88.03 {\scriptstyle \pm 1.46}$ & $88.34 {\scriptstyle \pm 3.92}$ & $68.76 {\scriptstyle \pm 2.76}$ & $39.76 {\scriptstyle \pm 0.30}$ & $64.05 {\scriptstyle \pm 13.68}$ & $72.94 {\scriptstyle \pm 4.28}$ \\
PGSO w/ $\mathbf{H}$ & $78.54 {\scriptstyle \pm 1.03}$ & $67.26 {\scriptstyle \pm 1.35}$ & $76.03 {\scriptstyle \pm 1.10}$ & $90.84 {\scriptstyle \pm 1.08}$ & $89.15 {\scriptstyle \pm 2.40}$ & $78.06 {\scriptstyle \pm 2.63}$ & $41.59 {\scriptstyle \pm 0.50}$ & $60.54 {\scriptstyle \pm 8.65}$ & $61.57 {\scriptstyle \pm 6.69}$ \\
PGSO w/ $\mathbf{\hat{A}}$ & $78.99 {\scriptstyle \pm 0.68}$ & $68.05 {\scriptstyle \pm 0.44}$ & $78.95 {\scriptstyle \pm 0.25}$ & $91.70 {\scriptstyle \pm 1.09}$ & $90.90 {\scriptstyle \pm 1.80}$ & $79.12 {\scriptstyle \pm 2.77}$ & $46.00 {\scriptstyle \pm 0.27}$ & $51.08 {\scriptstyle \pm 7.97}$ & $56.67 {\scriptstyle \pm 4.92}$ \\
PGSO w/ $\mathbf{L_{rw}}$ & $33.12 {\scriptstyle \pm 0.59}$ & $27.44 {\scriptstyle \pm 0.69}$ & $58.91 {\scriptstyle \pm 0.86}$ & $72.19 {\scriptstyle \pm 1.60}$ & $81.98 {\scriptstyle \pm 1.67}$ & $34.98 {\scriptstyle \pm 6.47}$ & $39.93 {\scriptstyle \pm 0.29}$ & $68.65 {\scriptstyle \pm 5.82}$ & $69.41 {\scriptstyle \pm 6.02}$ \\
PGSO w/ $\mathbf{Q}$ & $77.70 {\scriptstyle \pm 0.49}$ & $64.45 {\scriptstyle \pm 1.31}$ & $74.82 {\scriptstyle \pm 1.44}$ & $89.00 {\scriptstyle \pm 1.22}$ & $89.54 {\scriptstyle \pm 2.23}$ & $63.84 {\scriptstyle \pm 10.81}$ & $36.87 {\scriptstyle \pm 1.23}$ & $47.57 {\scriptstyle \pm 7.76}$ & $60.78 {\scriptstyle \pm 5.88}$ \\
PGSO w/ $\mathbf{L_{sym}}$ & $34.74 {\scriptstyle \pm 0.55}$ & $28.73 {\scriptstyle \pm 0.84}$ & $61.50 {\scriptstyle \pm 0.57}$ & $78.45 {\scriptstyle \pm 1.73}$ & $84.33 {\scriptstyle \pm 3.79}$ & $32.38 {\scriptstyle \pm 9.14}$ & $42.88 {\scriptstyle \pm 0.58}$ & $67.03 {\scriptstyle \pm 4.49}$ & $72.16 {\scriptstyle \pm 4.45}$ \\
PGSO w/ $\mathbf{L}$ & $22.60 {\scriptstyle \pm 9.61}$ & $27.30 {\scriptstyle \pm 0.88}$ & $41.37 {\scriptstyle \pm 1.70}$ & $26.49 {\scriptstyle \pm 3.03}$ & $42.19 {\scriptstyle \pm 6.64}$ & $24.32 {\scriptstyle \pm 2.20}$ & $31.77 {\scriptstyle \pm 0.40}$ & $72.43 {\scriptstyle \pm 2.65}$ & $79.02 {\scriptstyle \pm 3.93}$ \\ \midrule
\textbf{Detected by MSD} & $\mathbf{79.30 \pm 0.65}$ & $\mathbf{68.05 \pm 0.44}$ & $\mathbf{78.95 \pm 0.25}$ & $\mathbf{91.70 \pm 1.09}$ & $\mathbf{90.90 \pm 1.80}$ & $\mathbf{79.12 \pm 2.77}$ & $\mathbf{46.00 \pm 0.27}$ & $\mathbf{72.43 \pm 2.65}$ & $\mathbf{79.02 \pm 3.93}$ \\ \bottomrule
\end{tabular}%
}
\end{table*}

\end{document}